\documentclass[letterpaper, 10 pt, conference]{ieeeconf}
\IEEEoverridecommandlockouts
\usepackage{graphicx}
\usepackage{amsmath}
\usepackage{amssymb}

\usepackage{amsthm}
\usepackage{hyperref}
\usepackage{subcaption}
\usepackage{booktabs}
\providecommand{\citet}[1]{\cite{#1}}

\newcommand{\decode}{\mathrm{Dec}}
\newcommand{\encode}{\mathrm{Enc}}
\newcommand{\spaars}{\texttt{SPAARS}}

\newtheorem{proposition}{Proposition}

\newtheorem{remark}{Remark}
\newtheorem{corollary}{Corollary}
\newtheorem{assumption}{Assumption}
\newtheorem{definition}{Definition}

\title{\spaars: \underline{S}afer RL \underline{P}olicy \underline{A}lignment through \underline{A}bstract Exploration and \underline{R}efined Exploitation of Action \underline{S}pace}

\author{Swaminathan S K$^{1}$ and Aritra Hazra$^{2}$
\thanks{$^{1}$Swaminathan S K is with the Department of Computer Science and Engineering, IIT Kharagpur
        {\tt\small swami2004@kgpian.iitkgp.ac.in}}%
\thanks{$^{2}$Aritra Hazra is with the Department of Computer Science and Engineering, IIT Kharagpur
{\tt\small aritrah@cse.iitkgp.ac.in}}%
}

\begin{document}
\maketitle
\thispagestyle{empty}
\pagestyle{empty}

\begin{abstract}
Offline-to-online reinforcement learning (RL) offers a promising paradigm for robotics by pre-training policies on safe, offline demonstrations and fine-tuning them via online interaction. However, a fundamental challenge remains: {\em how to safely explore online without deviating from the behavioral support of the offline data?} While recent methods leverage conditional variational autoencoders (CVAEs) to bound exploration within a latent space, they inherently suffer from an \emph{exploitation gap} -- a performance ceiling imposed by the decoder's reconstruction loss. We introduce \spaars, a curriculum learning framework that initially constrains exploration to the low-dimensional latent manifold for sample-efficient, safe behavioral improvement, then seamlessly transfers control to the raw action space, bypassing the decoder bottleneck. \spaars\ has two instantiations: the CVAE-based variant requires only unordered $(s,a)$ pairs and no trajectory segmentation; \spaars-SUPE pairs \spaars\ with OPAL temporal skill pretraining for stronger exploration structure at the cost of requiring trajectory chunks. We prove an upper bound on the exploitation gap using the Performance Difference Lemma, establish that latent-space policy gradients achieve provable variance reduction over raw-space exploration, and show that concurrent behavioral cloning during the latent phase directly controls curriculum transition stability. Empirically, \spaars-SUPE achieves 0.825 normalized return on kitchen-mixed-v0 versus 0.75 for SUPE, with 5$\times$ better sample efficiency; standalone \spaars\ achieves \textbf{92.7} and \textbf{102.9} normalized return on hopper-medium-v2 and walker2d-medium-v2 respectively, surpassing IQL baselines of 66.3 and 78.3 respectively, confirming the utility of the unordered-pair CVAE instantiation.
\end{abstract}

\section{Introduction}

Reinforcement learning (RL) \cite{sutton_barto_book} has emerged as a powerful framework for solving complex sequential decision-making problems, demonstrating extraordinary success in practical domains ranging from playing strategic games to controlling stratospheric balloons and robotic manipulation \cite{mnih2015human, bellemare2020autonomous}. By actively interacting with an environment, RL agents discover optimal policies that maximize cumulative rewards. However, the applicability of pure online RL in real-world robotics is constrained by its sample inefficiency and the inherent physical risks of uncontrolled exploration.

To mitigate these challenges, Imitation Learning (IL) paradigms---including Apprenticeship Learning \cite{abbeel2004apprenticeship}, Inverse RL \cite{ziebart2008maximum}, and Behavioral Cloning (BC) \cite{pomerleau1988alvinn}---have been extensively studied \cite{hussein2017imitation}. These approaches extract policy behaviors directly from prior demonstrations. Recently, BC and offline RL have become state-of-the-art techniques for initializing robotic agents, enabling complex locomotion and manipulation policies without unsafe online interaction \cite{nair2020awac, kostrikov2021offline}.

While BC and purely offline methods provide robust, safe initializations, they are fundamentally limited by the quality and coverage of the provided dataset. If the offline dataset contains suboptimal demonstrations or lacks coverage in critical task areas, the resulting policy will asymptotically plateau at the performance level of the dataset. Consequently, an online fine-tuning phase is often necessary to push the policy toward true optimality. 

However, transitioning from offline imitation to online fine-tuning is notoriously unstable \cite{lee2022offline}. High-variance gradient updates outside the support of the offline data often cause catastrophic forgetting \cite{lee2022offline}. Recent offline-to-online algorithms address this by constraining online exploration to a learned latent skill space via Conditional Variational Autoencoders (CVAEs) \cite{zhou2020plas, shin2023supe}. While this grounds exploration within safe bounds, we identify a fundamental theoretical limitation in the form of an {\em exploitation gap}. Due to the reconstruction loss inherent to any dimensionality-reducing autoencoder, policies restricted to latent space cannot recover hyper-precise, optimally-tuned actions that exist in the raw action space.

To bridge this gap between safe latent exploration and optimally refined exploitation, we propose \spaars, a general framework with two concrete instantiations. The \textbf{standalone (CVAE-based)} instantiation trains a CVAE over unordered $(s,a)$ pairs, constrains online RL to the resulting latent space, and transitions to raw actions via a schedule or advantage gate---requiring no reward labels or trajectory segmentation. The \textbf{\spaars-SUPE} instantiation replaces the CVAE with OPAL~\cite{ajay2021opal} temporal $H$-step skills, inheriting the pretrained OPAL IQL policy for warm-started online fine-tuning; this delivers stronger exploration structure but requires trajectory chunks for skill pretraining. In both cases, rather than globally retiring the latent policy via a time-based schedule, \spaars\ introduces a \emph{state-dependent advantage gate} inspired by the Option-Critic architecture~\cite{bacon2017option}. The shared critic evaluates both policies at each decision point, activating the raw policy only where it demonstrably outperforms the decoder. A key design principle is that demonstrator alignment is not merely a constraint to overcome but a \emph{feature} -- the behavioral manifold which encodes (a) the actions that are physically coherent, (b) the joint configurations that are typical, and (c) the force magnitudes that are safe. Exploring within this manifold produces behaviors that are safe by construction and interpretable by design.

In summary, our primary contributions are:
\begin{itemize}
    \item \textbf{Theory.} We formally characterize the \emph{exploitation gap}, proving it is bounded by $O\!\big(\frac{L_Q \varepsilon_{\mathrm{rec}}}{1-\gamma}\big)$, that latent-space gradients achieve an $O\!\big(\frac{k}{d}\big)$ variance reduction, and that concurrent behavioral cloning during the latent phase directly controls curriculum transition stability (Propositions~1--4).
    \item \textbf{Algorithm.} We introduce \spaars, which resolves the exploitation gap via an \emph{advantage-gated mode selection} mechanism grounded in the Option-Critic termination gradient theorem~\cite{bacon2017option}. The shared critic selects per-state between latent and raw control---eliminating the catastrophic forgetting inherent in global schedules, without any new learned parameters. The CVAE-based instantiation requires only unordered $(s,a)$ pairs, unlike SUPE and OPAL which require trajectory segmentation.
    \item \textbf{Experiments.} We demonstrate that \spaars-SUPE outperforms SUPE on kitchen-mixed-v0 (0.825 vs.\ 0.75 normalized return), achieving 5$\times$ better sample efficiency by warm-starting from the pretrained OPAL policy, and that standalone \spaars\ achieves \textbf{92.7} and \textbf{102.9} normalized return on hopper-medium-v2 and walker2d-medium-v2 respectively---surpassing offline IQL baselines of 66.3 and 78.3 respectively---validating the unordered-pair CVAE instantiation across D4RL locomotion tasks.
\end{itemize}

The remainder of this paper is organized as follows: Section~\ref{sec:related_work} reviews the relevant work and the necessary background required for this work. Section~\ref{sec:method} presents the \spaars\ framework detailing the interplay between abstract latent space exploration and refined action space exploitation. Section~\ref{sec:experiments} illustrates our experimental evaluation. Section~\ref{sec:conclusion} concludes the work presented in the paper.

\section{Related Work}
\label{sec:related_work}

This work resides at the intersection of offline-to-online reinforcement learning and latent skill discovery. We leverage the stability of prior-constrained offline learning while proposing a dynamic curriculum to bypass representation bottlenecks during online fine-tuning.

\subsection{Offline-to-Online Reinforcement Learning}
Offline RL algorithms seek to learn effective policies from static datasets without environment interaction. Value-based methods typically apply pessimism to out-of-distribution (OOD) actions, using techniques such as conservative Q-learning \cite{kumar2020conservative} or implicit Q-learning (IQL) \cite{kostrikov2021offline}, which avoids querying OOD actions entirely by employing expectile regression. While highly effective on static datasets, directly fine-tuning these offline-trained policies online frequently leads to a severe drop in performance, colloquially termed ``catastrophic forgetting'' or the ``offline-to-online dip'' \cite{nair2020awac}. 

Methods like AWAC \cite{nair2020awac} and RLPD \cite{ball2023rlpd_main} attempt to smooth this transition. AWAC formulates an actor-critic algorithm that incorporates an implicit behavioral cloning penalty, while RLPD aggressively re-utilizes the offline replay buffer with high update-to-data ratios during online exploration. However, when exploring directly in the high-dimensional raw action space, these methods often struggle with high variance and unguided exploration, specially in sparse-reward settings. \spaars\ circumvents this initial high variance by strictly confining early exploration to a compressed behavioral manifold.

Table \ref{tab:comparison} provides a comparative footprint on the offline-to-online RL methods. Here, the $\checkmark$ symbol indicates that an approach natively supports a particular feature, $\times$ indicates that it does not, $-$ represents an inapplicable feature, and $\sim$ signifies partial support.

\subsection{Latent Action Spaces and Skill Discovery}
To improve exploratory efficiency and reduce the effective horizon of tasks, several works extract temporal or structural abstractions from offline data. OPAL \cite{ajay2021opal} and SPiRL \cite{pertsch2021spirl} learn continuous embeddings of trajectory snippets (skills) using hierarchical variational autoencoders, while PLAS \cite{zhou2020plas} constrains single-step action selections to the generative prior of a CVAE. 

Recently, the SUPE (State-Uniform Policy Evaluation) framework \cite{shin2023supe} unified these paradigms by utilizing OPAL encoders to provide rich trajectory primitives for offline RL evaluation. Operating entirely within these learned latent spaces provides exceptional stability and directs exploration exclusively along axes previously validated by human or expert demonstrators. However, as our theoretical analysis demonstrates, optimizing strictly within a CVAE latent space introduces a fundamental exploitation ceiling. No matter how thoroughly the latent space is explored, the policy can never execute behaviors finer than the decoder's reconstruction loss. \spaars\ builds directly upon the exploratory strengths of these latent spaces but introduces the necessary ``bridge'' back to the raw action space for optimal fine-tuning.

\begin{table}[t]
\caption{Comparison of offline-to-online RL methods.}
\label{tab:comparison}
\centering
\footnotesize
\begin{tabular}{l@{\hspace{4pt}}c@{\hspace{4pt}}c@{\hspace{4pt}}c@{\hspace{4pt}}c}
\hline
\textbf{Method} & \textbf{BC} & \textbf{Safe} & \textbf{Optimal} & \textbf{$k$-RND} \\
\hline
IQL       & \checkmark & ---        & \texttimes & --- \\
PLAS      & \checkmark & \checkmark & \texttimes & --- \\
ExPLORe   & \checkmark & \texttimes & \checkmark & \texttimes \\
SUPE      & \texttimes & $\sim$     & \checkmark & \texttimes \\
\\
\textbf{\spaars} & \checkmark & \checkmark & \checkmark & \checkmark \\
\hline
\end{tabular}
\end{table}

\section{Abstract Exploration and Refined Exploitation Framework}
\label{sec:method}

In this section, we formulate the mathematical framework underpinning \spaars. We demonstrate why exploration within a CVAE latent space reduces policy gradient variance, formalize the exploitation ceiling imposed by the decoder, and present the \spaars\ curriculum that can bridge this gap.

\subsection{Preliminaries and Notations}
We consider a Markov Decision Process $(\mathcal{S}, \mathcal{A}, P, R, \gamma)$ with continuous action space $\mathcal{A} \subseteq \mathbb{R}^d$ and discount factor $\gamma \in (0,1)$. Let $\mathcal{D} = \{(s_i, a_i)\}_{i=1}^{N}$ be a static offline dataset collected by a behavioral policy $\pi_\beta$. We train a Conditional Variational Autoencoder (CVAE) over the action space, mapping raw actions to a compressed latent manifold $\mathcal{Z} \subseteq \mathbb{R}^k$ where $k < d$. The CVAE consists of an encoder $q_\phi(z \mid s, a)$, a decoder $p_\theta(a \mid s, z)$ with deterministic mode $\decode(z, s)$, and a learned prior $p_\psi(z \mid s)$.

The decoder defines the \emph{behavioral action manifold} $\mathcal{M}_a = \{\decode(z, s) : z \in \mathcal{Z}, s \in \mathcal{S}\} \subset \mathcal{A}$. Throughout, we denote the optimal unconstrained policy as $\pi_a^* = \arg\max_{\pi: \mathcal{S} \to \mathcal{A}} J(\pi)$ and the optimal latent policy as $\pi_z^* = \arg\max_{\pi_z: \mathcal{S} \to \mathcal{Z}} J(\pi_z)$.

\begin{remark}[BC Dataset Sufficiency]
\label{rmk:bc_data}
The CVAE training objective (ELBO) is an expectation over $(s, a)$ pairs---no trajectory structure, temporal ordering, or reward labels appear in the loss. Unlike SUPE and OPAL which require trajectory-segmented data, the CVAE-based \spaars\ instantiation operates on pure behavioral cloning datasets: unordered $(s, a)$ pairs suffice. The OPAL-based \spaars-SUPE instantiation (Section~\ref{subsec:spaars_supe}) requires trajectory chunks for OPAL skill pretraining, though the online fine-tuning stage is dataset-format agnostic once skills are learned.
\end{remark}

\subsection{Gradient Variance Reduction in Latent Space}
By restricting the online agent to query actions via $a = \decode(\pi_z(s), s)$, exploration is forced to remain on $\mathcal{M}_a$. This provides a fundamental variance reduction during the critical early phases of online fine-tuning.

\begin{proposition}[Variance Reduction Lemma]
\label{prop:variance}
Let $\pi_z(z \mid s) = \mathcal{N}(\mu_\theta(s), \sigma^2 I_k)$ and $\pi_a(a \mid s) = \mathcal{N}(\nu_\theta(s), \sigma^2 I_d)$ be isotropic Gaussian policies in latent and raw action spaces respectively. The variance of the REINFORCE gradient estimator satisfies:
\begin{equation}
    \frac{\mathrm{Var}[\nabla_\theta J_z]}{\mathrm{Var}[\nabla_\theta J_a]} \leq \frac{k}{d} \cdot \frac{\mathrm{Var}_z[Q(s, \decode(z,s))]}{\mathrm{Var}_a[Q(s,a)]}.
\end{equation}
\end{proposition}

\begin{proof}
The REINFORCE gradient is $\nabla_\theta J = \mathbb{E}_{s,a}[Q(s,a) \nabla_\theta \log \pi(a \mid s)]$. For isotropic Gaussian $\pi(a \mid s) = \mathcal{N}(\mu_\theta(s), \sigma^2 I_n)$ in $\mathbb{R}^n$, the score function is $\nabla_\theta \log \pi = \sigma^{-2}(a - \mu_\theta(s))^\top \nabla_\theta \mu_\theta(s)$, whose second moment satisfies $\mathbb{E}[\|\nabla_\theta \log \pi\|^2] \propto n / \sigma^4$. The variance of the full gradient decomposes as $\mathrm{Var}[\nabla_\theta J] \approx \mathbb{E}[Q^2] \cdot \mathrm{Var}[\nabla_\theta \log \pi] + \mathbb{E}[\|\nabla_\theta \log \pi\|^2] \cdot \mathrm{Var}[Q]$. Comparing $n = k$ (latent) with $n = d$ (raw) yields the $k/d$ dimensional factor. The Q-variance ratio captures the remaining discrepancy.
\end{proof}

Under the empirical observation that $\mathrm{Var}_z[Q(s, \decode(z,s))] \leq \mathrm{Var}_a[Q(s,a)]$---which holds because the decoder restricts outputs to $\mathcal{M}_a$, excluding physically incoherent OOD actions with extreme Q-values---the full bound becomes $\mathrm{Var}[\nabla_\theta J_z] \leq (k/d) \cdot \mathrm{Var}[\nabla_\theta J_a]$.

\begin{remark}[Gradient Projection]
\label{rmk:jacobian}
The decoder Jacobian $J_{\decode} \in \mathbb{R}^{d \times k}$ has rank at most $k < d$. The latent Q-gradient $\nabla_z \tilde{Q}(s,z) = J_{\decode}^\top \nabla_a Q(s,a)\big|_{a = \decode(z,s)}$ is therefore a projection onto a $k$-dimensional subspace aligned with the tangent space of $\mathcal{M}_a$. The $(d{-}k)$ gradient directions orthogonal to this subspace---corresponding to off-manifold action perturbations---are automatically filtered out.
\end{remark}

\subsection{The Exploitation Gap}
While latent exploration provides superior gradient-signal properties, its lower dimensionality imposes a strict structural constraint on the agent's maximum attainable performance.

\begin{proposition}[Exploitation Gap Bound]
\label{prop:expgap}
Let $Q^*$ be $L_Q$-Lipschitz in actions. Let $\varepsilon_{\mathrm{rec}} = \mathbb{E}_{(s,a)\sim\mathcal{D}}[\|a - \decode(\encode(s,a),s)\|^2]^{1/2}$ be the CVAE reconstruction error. Under Assumption~\ref{ass:coverage}, the exploitation gap satisfies:
\begin{equation}
    \Delta_{\mathrm{exploit}} := J(\pi_a^*) - J(\pi_z^*) \leq \frac{L_Q \, \varepsilon_{\mathrm{rec}}}{1 - \gamma}.
\end{equation}
\end{proposition}

\begin{assumption}[Coverage]
\label{ass:coverage}
The offline dataset $\mathcal{D}$ provides sufficient coverage of near-optimal actions: $\mathrm{supp}(d_{\pi^*}) \subseteq \mathrm{supp}(d_{\pi_\beta})$, so that $\varepsilon_{\mathrm{rec}}$ approximates the reconstruction error under $\pi^*$'s visitation distribution.
\end{assumption}

\begin{proof}
By the Performance Difference Lemma~\cite{kakade2002approximately}: $J(\pi') - J(\pi) = \frac{1}{1-\gamma} \mathbb{E}_{s \sim d_{\pi'}}[\mathbb{E}_{a \sim \pi'}[A^\pi(s,a)]]$. For any state $s$, let $a^* = \pi_a^*(s)$ and $\tilde{a} = \decode(\encode(a^*,s),s)$ be its reconstruction. Since $z^* = \arg\max_z Q^*(s, \decode(z,s))$ is the optimal latent action, $Q^*(s, \decode(z^*,s)) \geq Q^*(s, \tilde{a})$. By Lipschitz continuity: $Q^*(s, a^*) - Q^*(s, \decode(z^*,s)) \leq Q^*(s, a^*) - Q^*(s, \tilde{a}) \leq L_Q \|a^* - \tilde{a}\|$. Aggregating over $d_{\pi_a^*}$ and applying Assumption~\ref{ass:coverage} yields the result.
\end{proof}

\begin{remark}[Dataset Quality]
The bound is tight when $\mathcal{D}$ contains near-optimal demonstrations. For suboptimal datasets, $\varepsilon_{\mathrm{rec}}$ may be large for actions outside behavioral support, yielding a larger gap. This formalizes why Phase~2 ($\alpha \to 1$) is necessary: the latent space inherits the demonstrator's performance ceiling.
\end{remark}

\begin{remark}[Worst-Case Alternative]
When coverage (Assumption~\ref{ass:coverage}) cannot be verified, a distribution-free bound holds using the supremum reconstruction error $\varepsilon_{\mathrm{rec}}^{\sup} = \sup_{s,a} \big{\|}a - \decode(\encode(a,s),s) \big{\|}$:
\begin{equation}
    \Delta_{\mathrm{exploit}} \leq \frac{L_Q \, \varepsilon_{\mathrm{rec}}^{\sup}}{1 - \gamma}.
\end{equation}
This bound is looser but does not depend on distributional assumptions.
\end{remark}

\subsection{The \spaars\ Curriculum}
To achieve asymptotic optimality, \spaars\ utilizes a deterministic mixing schedule $\alpha(t) \in [0, 1]$ to govern control between the latent policy $\pi_z(s)$ and a separately trained raw policy $\pi_{\mathrm{raw}}(s)$. The executed action is:
\begin{equation}
    a(t) = (1 - \alpha(t)) \cdot \decode(\pi_z(s), s) + \alpha(t) \cdot \pi_{\mathrm{raw}}(s).
\end{equation}

\subsubsection*{Phase 1: Latent Exploration ($\alpha = 0$).}
The policy operates solely within $\mathcal{Z}$. RND intrinsic rewards over the latent space promote maximal state-space coverage. Both $\pi_z$ and $\pi_{\mathrm{raw}}$ train concurrently from a shared replay buffer $\mathcal{B}$: while $\pi_z$ is updated via RL in latent space, $\pi_{\mathrm{raw}}$ trains via behavioral cloning on the same buffer, minimizing $\mathbb{E}_{(s,a) \sim \mathcal{B}}[\|\pi_{\mathrm{raw}}(s) - a\|^2]$. This ensures that $\pi_{\mathrm{raw}}$ is distributionally aligned with $\decode(\pi_z(s), s)$ before blending begins. The phase concludes when (i)~the RND plateau (Definition~\ref{def:rnd_plateau}) signals exhaustion of latent exploration, and (ii)~the BC loss $\mathcal{L}_{\mathrm{BC}} < \varepsilon_{\mathrm{BC}}$, ensuring $\pi_{\mathrm{raw}}$ is competent.

\begin{definition}[RND Plateau]
\label{def:rnd_plateau}
The RND intrinsic reward has plateaued at time $T$ if:
\begin{equation}
    \frac{|\bar{r}_{\mathrm{int}}(T) - \bar{r}_{\mathrm{int}}(T - W)|}{|\bar{r}_{\mathrm{int}}(T - W)|} < \tau
\end{equation}
where $\bar{r}_{\mathrm{int}}(t)$ is the exponential moving average of episodic intrinsic reward and $W$ is the window size. We use $\tau = 0.01$ (1\% relative change).
\end{definition}

\subsubsection*{Phase 2: Raw Exploitation ($\alpha \to 1$).}
Upon exhausting latent exploration, the schedule fades $\alpha$ from 0 to 1. Because $\pi_{\mathrm{raw}}$ has been pre-aligned via BC during Phase~1, the initial blending operates safely within the local neighborhood of $\mathcal{M}_a$. As $\alpha \to 1$, $\pi_{\mathrm{raw}}$ bypasses the $\varepsilon_{\mathrm{rec}}$ constraint, enabling unconstrained exploitation to reach $\max_{\pi_a} J(\pi_a)$.

\begin{proposition}[Calibration Stability Lemma]
\label{prop:calibration}
Let $Q$ be $L_Q$-Lipschitz and $(\delta, \varepsilon_\mathcal{M})$-calibrated on $\mathcal{M}_a$, i.e., $\Pr[|Q(s,a) - Q^*(s,a)| > \varepsilon_\mathcal{M}] \leq \delta$ for $a \in \mathcal{M}_a$. Then for the blended action $a(\alpha) = (1{-}\alpha)\,a_z + \alpha\,a_r$:
\begin{equation}
    |Q(s, a(\alpha)) - Q^*(s, a(\alpha))| \leq \varepsilon_\mathcal{M} + L_Q \cdot \alpha \cdot \|a_r - a_z\|
\end{equation}
with probability $\geq 1 - \delta$.
\end{proposition}

\begin{proof}
At $\alpha = 0$, $a(0) = a_z \in \mathcal{M}_a$, so $|Q(s,a_z) - Q^*(s,a_z)| \leq \varepsilon_\mathcal{M}$ by calibration. For $\alpha > 0$, $\|a(\alpha) - a_z\| = \alpha \|a_r - a_z\|$. Lipschitz continuity of $Q^*$ gives $|Q^*(s,a(\alpha)) - Q^*(s,a_z)| \leq L_Q \alpha \|a_r - a_z\|$, and similarly for the learned $Q$. The triangle inequality yields the result.
\end{proof}

\begin{proposition}[Transition Smoothness Bound]
\label{prop:transition}
Let $\pi_{\mathrm{raw}}$ achieve BC loss $\varepsilon_{\mathrm{BC}}$ on the shared buffer during Phase~1. Then the blended policy's deviation from the manifold satisfies:
\begin{equation}
    \mathbb{E}_{s \sim d_{\pi_z}}\!\bigl[\|\pi_\alpha(s) - \decode(\pi_z(s),s)\|\bigr] \leq \alpha \sqrt{\varepsilon_{\mathrm{BC}}}.
\end{equation}
\end{proposition}

\begin{proof}
$\|\pi_\alpha(s) - \decode(\pi_z(s),s)\| = \alpha \|\pi_{\mathrm{raw}}(s) - \decode(\pi_z(s),s)\|$. During Phase~1, $\mathcal{B}$ stores $(s, a)$ with $a = \decode(\pi_z(s),s)$, so $\mathbb{E}[\|\pi_{\mathrm{raw}}(s) - a\|^2] = \varepsilon_{\mathrm{BC}}$. Jensen's inequality gives $\mathbb{E}[\|\pi_{\mathrm{raw}}(s) - a\|] \leq \sqrt{\varepsilon_{\mathrm{BC}}}$.
\end{proof}

\begin{corollary}
Combining Propositions~\ref{prop:calibration} and~\ref{prop:transition}: $|Q(s, a(\alpha)) - Q^*(s, a(\alpha))| \leq \varepsilon_\mathcal{M} + L_Q \cdot \alpha \cdot \sqrt{\varepsilon_{\mathrm{BC}}}$, with high probability. Thus BC training during Phase~1 directly controls curriculum stability.
\end{corollary}

\subsubsection*{The Overall Algorithm.}
The complete \spaars\ training procedure is given below.

\begin{table}[h]
\centering
\begin{tabular}{l}
\hline
\textbf{Algorithm 1: \spaars\ Training} \\
\hline
\textbf{Input:} Offline dataset $\mathcal{D}$, CVAE architecture, RL algorithm \\
\textbf{Output:} Trained policy $\pi_{\mathrm{raw}}$ \\
\hline
\textit{// Phase 0: CVAE Pretraining} \\
1. Train CVAE $(q_\phi, p_\theta, p_\psi)$ on $\mathcal{D}$ via ELBO; freeze $\decode$. \\
\\
\textit{// Phase 1: Latent Exploration ($\alpha = 0$)} \\
2. Initialize $\pi_z$, $\pi_{\mathrm{raw}}$, critic $Q$, shared buffer $\mathcal{B}$. \\
3. \textbf{while} RND\_plateau \textbf{or} $\mathcal{L}_{\mathrm{BC}} \le \varepsilon_{\mathrm{BC}}$, {\bf do}: \\
\qquad a. Execute $a = \decode(\pi_z(s), s)$.\\
\qquad b. Store $(s,a,r,s')$ in $\mathcal{B}$. \\
\qquad c. Update $\pi_z$ via SAC in $z$-space using $\mathcal{B}$ + RND bonus. \\
\qquad d. Update $\pi_{\mathrm{raw}}$ via BC: $\min \|\pi_{\mathrm{raw}}(s) - a\|^2$ on $\mathcal{B}$. \\
\qquad e. Update $Q$ using $\mathcal{B}$. \\
\\
\textit{// Phase 2: Curriculum Transition ($\alpha: 0 \to 1$)} \\
4. \textbf{while} $\alpha$ in schedule$(0 \to 1)$, {\bf do}: \\
\qquad a. Execute $a = (1{-}\alpha)\decode(\pi_z(s),s) + \alpha\,\pi_{\mathrm{raw}}(s)$. \\
\qquad b. Store $(s,a,r,s')$ in $\mathcal{B}$. \\
\qquad c. Update $\pi_{\mathrm{raw}}$ via SAC in action space using $\mathcal{B}$. \\
\qquad d. Update $Q$ using $\mathcal{B}$. \\
\\
\textit{// Phase 3: Raw Exploitation ($\alpha = 1$)} \\
5. Continue training $\pi_{\mathrm{raw}}$ via SAC until convergence. \\
\\
\textbf{return} policy $\pi_{\mathrm{raw}}$ \\
\hline
\end{tabular}
\end{table}
The algorithm above describes the \textbf{Schedule Variant} ($\alpha: 0 \to 1$). For the \textbf{Gate Variant} (Section~\ref{subsec:gate}), step 4a is replaced by the gate evaluation in Eq.~\eqref{eq:gate_condition}: rather than blending proportional to $\alpha$, the gate commits fully to either $\pi_z$ or $\pi_{\mathrm{raw}}$ for the entire $H$-step window based on the per-state exploitation advantage. The gate variant never enters Phase~3 (there is no global $\alpha = 1$); both policies remain active indefinitely. For standalone \spaars\ (CVAE-based), the gate decision is made per step rather than per $H$-step window since there is no temporal skill commitment.

\subsection{\spaars-SUPE: Instantiation with Temporal Skills}
\label{subsec:spaars_supe}

The CVAE-based formulation above requires only unordered $(s,a)$ pairs and provides single-step skill representations. A more powerful instantiation, \textbf{\spaars-SUPE}, replaces the CVAE with OPAL~\cite{ajay2021opal} temporal skill pretraining, gaining richer exploration structure at the cost of requiring trajectory chunks.

\subsubsection*{OPAL as the Skill Encoder.}
OPAL trains a BiGRU encoder $q_\phi(z \mid \tau_{1:H})$ that maps $H$-step trajectory chunks $\tau_{1:H} = (s_1, a_1, \ldots, s_H, a_H)$ to a low-dimensional skill embedding $z \in \mathcal{Z} \subset \mathbb{R}^k$, and a step-wise decoder $p_\theta(a_t \mid s_t, z)$ that generates actions conditioned on the current state and skill. After pretraining, OPAL IQL~\cite{kostrikov2021offline} is used to learn a latent policy $\pi_z(z \mid s)$ and a value function $V(s)$ entirely in $\mathcal{Z}$, providing a warm-started agent before any online interaction.

\subsubsection*{Compatibility with the Theoretical Framework.}
The latent skill space $\mathcal{Z}$ is directly analogous to the CVAE latent in Section~\ref{sec:method}: $k = \mathrm{skill\_dim} \ll d$, so the variance-reduction bound (Proposition~\ref{prop:variance}) and exploitation gap bound (Proposition~\ref{prop:expgap}) apply to the OPAL decoder $p_\theta$ in place of $p_\theta^{\mathrm{CVAE}}$.

\subsubsection*{Shared Critic and Warm-Start Advantage.}
The key distinction from SUPE is that \spaars-SUPE uses a \emph{shared critic} $Q(s, a)$ that evaluates raw actions from both $\pi_z$ (OPAL IQL decoded) and $\pi_{\mathrm{raw}}$ (SAC actor) in the same action space. SUPE uses a separate $Q(s, z)$ critic over the skill space, which cannot directly compare latent and raw policy quality. The shared critic enables the advantage gate (Eq.~\ref{eq:gate_condition}) and eliminates the cold-start period: because $\pi_z$ inherits the pretrained OPAL IQL policy, the agent begins online training with a competent latent policy rather than a random one.

\subsubsection*{$H$-Step Commitment.}
At each $H$-step boundary, the agent commits to either $\pi_z$ or $\pi_{\mathrm{raw}}$ for the full $H$-step window. This temporal commitment is consistent with the options framework underpinning the gate (Section~\ref{subsec:gate}): the $H$-step window defines the option duration, and the gate's termination decision occurs at each boundary.

\subsection{Advantage-Gated Mode Selection}
\label{subsec:gate}

The $\alpha$ schedule described above is a global, time-based curriculum: once $\alpha \to 1$, the latent policy $\pi_z$ is permanently retired. This creates a dilemma. The schedule must eventually reach $\alpha = 1$ to close the exploitation gap, but doing so destroys the temporal abstraction and exploration capabilities that $\pi_z$ provides. In long-horizon navigation tasks, $\pi_z$ is strictly superior in states far from the goal---the decoder ceiling only matters in states where precise, task-specific actions are needed.

We resolve this with a \emph{state-dependent} mode selection mechanism inspired by the Option-Critic architecture~\cite{bacon2017option}. Rather than globally retiring $\pi_z$, we let the shared critic $Q(s,a)$ decide \emph{per state} whether raw or latent control is superior.

\subsubsection*{Theoretical Motivation.}
In the options framework~\cite{sutton1999options}, a set of options $\Omega = \{\omega_1, \ldots, \omega_n\}$ execute sub-policies with termination conditions $\beta_\omega(s)$. The Option-Critic Termination Gradient Theorem~\cite{bacon2017option} states:
\begin{equation}
\label{eq:termination_gradient}
    \frac{\partial J}{\partial \vartheta} = -\sum_{s', \omega} \mu_\Omega(s', \omega) \frac{\partial \beta_{\omega, \vartheta}(s')}{\partial \vartheta} A_\Omega(s', \omega)
\end{equation}
where $A_\Omega(s', \omega) = Q_\Omega(s', \omega) - V_\Omega(s')$ is the advantage of continuing with option $\omega$, and $\mu_\Omega$ is the discounted option-state visitation. The gradient increases termination probability when continuing with the current option is disadvantageous.

In our setting, we instantiate this with two options: $\omega_z$ (latent skill, $H$-step committed) and $\omega_{\mathrm{raw}}$ (raw policy, $H$-step committed). Rather than learning a parameterized termination function, we observe that the shared critic $Q(s, a)$ already provides all the information needed to evaluate both options. The \emph{exploitation advantage} at state $s$ is:
\begin{equation}
\label{eq:exploitation_advantage}
    A_{\mathrm{exploit}}(s) = Q(s, \pi_{\mathrm{raw}}(s)) - Q(s, \decode(\pi_z(s), s)).
\end{equation}
When $A_{\mathrm{exploit}}(s) > 0$, the raw policy closes the decoder ceiling at $s$. When $A_{\mathrm{exploit}}(s) \leq 0$, the latent policy is superior---its temporal abstraction and exploratory structure outweigh any decoder loss.

\begin{proposition}[State-Dependent Regret Bound]
\label{prop:gate_regret}
Let $g: \mathcal{S} \to \{z, \mathrm{raw}\}$ be a deterministic gate that selects $\pi_{\mathrm{raw}}$ when $A_{\mathrm{exploit}}(s) > 0$ and $\pi_z$ otherwise. The total regret of the gated policy $\pi_g$ relative to the oracle that selects the best option per state satisfies:
\begin{equation}
    J(\pi^*_{\mathrm{oracle}}) - J(\pi_g) \leq \frac{1}{1-\gamma} \mathbb{E}_{s \sim d_{\pi_g}}\!\bigl[\varepsilon_Q(s)\bigr]
\end{equation}
where $\varepsilon_Q(s) = |Q(s,a) - Q^*(s,a)|$ is the critic approximation error.
\end{proposition}

\begin{proof}
At each state $s$, the gate selects the option with higher estimated Q-value. If $Q = Q^*$ (perfect critic), then $\pi_g(s) = \arg\max\{Q^*(s, \pi_{\mathrm{raw}}(s)), Q^*(s, \decode(\pi_z(s),s))\}$, which matches the oracle. Errors arise solely from critic misestimation. By the Performance Difference Lemma, the regret is bounded by the expected per-state advantage error, which is at most $\varepsilon_Q(s)$ per step, aggregated over the discounted visitation.
\end{proof}

\begin{remark}[Comparison with Global Schedule]
Under the $\alpha$ schedule, the total regret includes a term $\frac{1}{1-\gamma}\mathbb{E}_{s}[\mathbf{1}\{g^*(s) = z\} \cdot \Delta_{\mathrm{forget}}(s)]$ from states where $\pi_z$ would have been superior but was retired globally. The gate eliminates this term entirely: $\pi_z$ remains active wherever it is advantageous. The gate's regret depends only on critic quality $\varepsilon_Q$, not on schedule design.
\end{remark}

\subsubsection*{Gate Mechanism.}
At each $H$-step decision boundary, the gate evaluates:
\begin{enumerate}
    \item Sample $z \sim \pi_z(s)$, decode $a_z = \decode(z, s)$.
    \item Sample $a_{\mathrm{raw}} \sim \pi_{\mathrm{raw}}(s)$.
    \item Compute ensemble statistics: $\bar{Q}_{\mathrm{raw}} = \frac{1}{K}\sum_k Q_k(s, a_{\mathrm{raw}})$, \; $\bar{Q}_z = \frac{1}{K}\sum_k Q_k(s, a_z)$, \; $\sigma_{\mathrm{raw}} = \mathrm{std}_k(Q_k(s, a_{\mathrm{raw}}))$.
    \item Gate fires (use $\pi_{\mathrm{raw}}$) if:
    \begin{equation}
    \label{eq:gate_condition}
        \bar{Q}_{\mathrm{raw}} - \bar{Q}_z > m \quad \text{and} \quad \sigma_{\mathrm{raw}} < \sigma_{\max}
    \end{equation}
    where $m \geq 0$ is a margin (biasing toward $\pi_z$) and $\sigma_{\max}$ is a disagreement threshold.
\end{enumerate}

No new parameters are learned. The shared critic \emph{is} the gate---this is the cleanest instantiation of the Option-Critic insight that the advantage function naturally drives option termination.

\subsubsection*{Safeguards Against Premature Switching.}
Three mechanisms prevent the gate from activating $\pi_{\mathrm{raw}}$ before the critic is calibrated:

\begin{itemize}
    \item \textbf{Warmup lockout.} During the first $T_{\mathrm{warm}}$ steps after training begins, the gate is disabled and $\pi_z$ (via the pre-trained IQL actor) drives all rollouts. This allows the critic to observe both action types before making decisions.
    \item \textbf{Margin $m \geq 0$.} The raw policy must be \emph{meaningfully} better, not just marginally. This biases toward the temporal abstraction of $\pi_z$.
    \item \textbf{Ensemble disagreement filter $\sigma_{\max}$.} If the critic ensemble disagrees about $Q(s, a_{\mathrm{raw}})$---indicating the state-action pair is outside the training distribution---the gate defaults to $\pi_z$. This prevents activation in unexplored states where Q-estimates are unreliable.
\end{itemize}

\begin{proposition}[Gate Convergence]
\label{prop:gate_convergence}
As the critic converges ($\varepsilon_Q \to 0$), the gate activation set $\mathcal{S}_{\mathrm{raw}} = \{s : A_{\mathrm{exploit}}(s) > m\}$ converges to the set of states where the exploitation gap is genuine. In the limit, $\pi_z$ is active everywhere except near-goal states where decoder precision matters.
\end{proposition}

This yields the ``best of both worlds'': $\pi_z$ provides temporally abstract, exploratory navigation in the vast majority of the state space, while $\pi_{\mathrm{raw}}$ closes the exploitation gap only where it provably exists. Unlike the $\alpha$ schedule, neither policy is globally retired.

\subsection{Formal Characterizations: Assumptions, Guarantees, and Tradeoffs}
\label{subsec:tradeoff}
The following subsections consolidate the key assumptions underlying the theoretical guarantees of \spaars\ and characterize the regimes in which they hold tightly. We also distinguish \spaars\ from closely related approaches that operate on or near the behavioral manifold.

\subsubsection*{Lipschitz Q-Function (Assumption).}
We assume $Q^*$ is $L_Q$-Lipschitz in actions: $|Q^*(s,a) - Q^*(s,a')| \leq L_Q \|a - a'\|$. This is standard in the RL theory literature~\cite{kakade2002approximately} but worth noting that learned neural Q-functions can have large Lipschitz constants in practice.

\subsubsection*{Safe Demonstrations (Guarantee Regime).}
The guarantees produced by \spaars\ are strongest when $\pi_\beta$ produces safe trajectories, i.e., demonstrations that do not terminate in catastrophic failure states. Formally, the safety guarantee holds whenever $\mathbb{E}_{a \sim \pi_\beta(\cdot \mid s)}[V^{\pi_\beta}(s)] \geq V_{\min}$ for some acceptable lower bound. Under this condition, the CVAE learns $\mathcal{M}_a$ consisting exclusively of physically coherent actions, and any $z \sim p_\psi(\cdot \mid s)$ decodes to a safe action by construction. Demonstrator alignment is not merely a constraint but a \emph{feature}: it encodes which actions are physically coherent, which joint configurations are typical, and which force magnitudes are safe.

\subsubsection*{Sub-optimal Demonstrations (Degraded Regime).}
When $\pi_\beta$ is sub-optimal (e.g., a ``medium'' D4RL dataset), $\mathcal{M}_a$ encodes the \emph{demonstrator's performance ceiling}. During Phase~1, no online RL gradient can instruct the agent to exceed the demonstrator's maximum observable joint torques. This is the fundamental tradeoff: alignment provides a safety floor at the cost of an exploration ceiling during Phase~1. Unconstrained methods (e.g., RLPD) do not face this ceiling but abandon the safety floor entirely.

\subsubsection*{Resolution via the Curriculum.}
\spaars\ resolves this exploration-exploitation tradeoff via Phase~2. Once the RND plateau signals exhaustion of $\mathcal{M}_a$, $\alpha \to 1$ releases $\pi_{\mathrm{raw}}$ to explore the full $\mathcal{A}$. \spaars\ therefore pays a \emph{delayed} exploration cost rather than an upfront safety cost. On expert datasets, \spaars\ dominates both in sample efficiency and final performance. On sub-optimal datasets, it provides a higher safety floor and lower training variance, at the cost of slower early-phase convergence.

\subsubsection*{Tightness of Bounds.}
The $\frac{1}{1-\gamma}$ factor in the exploitation gap (Proposition~\ref{prop:expgap}) is a worst-case horizon scaling: for $\gamma = 0.99$, it amplifies by $100\times$. In practice, empirical gaps are substantially smaller because (i)~the Lipschitz constant $L_Q$ is local, not global, and (ii)~reconstruction errors are small for well-trained CVAEs. We validate this empirically in Section~\ref{sec:experiments}.

\subsubsection*{Distinction from PLAS.}
PLAS~\cite{zhou2020plas} adds a learned perturbation $\xi_\phi(s)$ to decoded actions: $a = \decode(z, s) + \xi_\phi(s)$. These perturbations are small and local by design, remaining within the neighborhood of the manifold. In contrast, $\pi_{\mathrm{raw}}$ obtained from \spaars\ is a full policy that eventually \emph{replaces} the decoded output entirely ($\alpha \to 1$), enabling exploration of the complete raw action space $\mathcal{A}$.

The advantage gate (Section~\ref{subsec:gate}) sharpens this comparison further: whereas PLAS perturbations remain permanently bounded to a tube around $\mathcal{M}_a$, the gate's activation set $\mathcal{S}_{\mathrm{raw}}$ grows as the critic improves, progressively expanding access to the full action space $\mathcal{A}$ exactly where the gap warrants it.

\section{Experimental Evaluation}
\label{sec:experiments}

The empirical evaluation of \spaars\ is designed to validate its core theoretical predictions: (1)~warm-starting from the pretrained OPAL policy eliminates the cold-start period of SUPE, (2)~the advantage gate closes the exploitation gap above the pretrained baseline without globally retiring the latent policy, and (3)~the standalone CVAE-based instantiation, trained on unordered $(s,a)$ pairs, can improve upon the offline baseline during online fine-tuning.

\paragraph{Setup.}
We compare three agents: \textbf{SUPE}~\cite{shin2023supe} (best published config: \texttt{offline\_relabel\_type=pred}, \texttt{use\_rnd\_offline=True}, \texttt{num\_min\_qs=2}), \textbf{\spaars-SUPE (schedule)} ($\alpha: 0 \to 1$ ramp), and \textbf{\spaars-SUPE (gate)} (advantage gate, $m=3.0$, $\sigma_{\max}=10.0$). All \spaars-SUPE agents use OPAL pretrained on each environment's offline dataset, share the same offline relabeling config as SUPE, and run for 300k environment steps.

\subsection{Kitchen-Mixed-v0: Manipulation with Temporal Skills}
\label{subsec:kitchen}

\paragraph{Environment.}
D4RL kitchen-mixed-v0~\cite{fu2020d4rl}: the agent must complete 4 subtasks in sequence (microwave, kettle, bottom burner, light switch). Maximum return is 4.0; we report normalized return (divided by 4.0).

\paragraph{Results.}
\spaars-SUPE (gate) starts at approximately 0.70 normalized return at the beginning of online training, inheriting the pretrained OPAL IQL policy, while SUPE starts near 0.025 with no pretrained policy. By 300k environment steps, \spaars-SUPE reaches \textbf{0.825 normalized return} (3.3/4.0 tasks) while SUPE plateaus at \textbf{0.75 normalized return} (3.0/4.0 tasks). \spaars-SUPE achieves SUPE's asymptotic performance in under 50k steps versus approximately 250k steps for SUPE---a \textbf{5$\times$ sample efficiency} improvement attributable entirely to the warm-start from OPAL IQL.

\paragraph{Gate Behavior.}
The exploitation advantage $A_{\mathrm{exploit}}$ oscillates between 0 and 4 throughout training. The gate fires non-trivially at all training stages, confirming that both $\pi_z$ and $\pi_{\mathrm{raw}}$ remain active---neither policy is globally retired. This is consistent with Proposition~\ref{prop:gate_convergence}: the gate concentrates $\pi_{\mathrm{raw}}$ in states where the decoder ceiling is binding, while $\pi_z$ retains control elsewhere.

\begin{figure}[h]
\centering
\includegraphics[width=\linewidth]{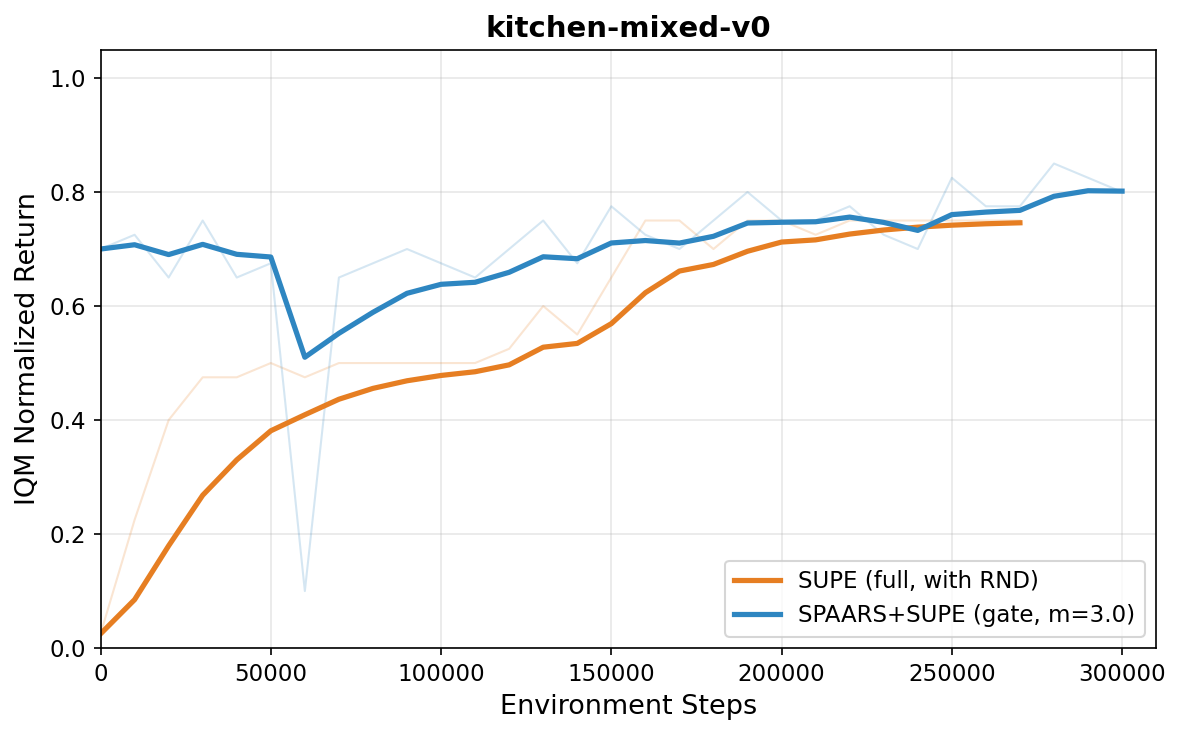}
\caption{Normalized return vs.\ environment steps on kitchen-mixed-v0. \spaars-SUPE (gate) warm-starts from the pretrained OPAL IQL policy, reaching SUPE's asymptotic performance 5$\times$ faster and surpassing it by 300k steps.}
\label{fig:kitchen}
\end{figure}

These results confirm the core theoretical prediction: warm-starting from the pretrained OPAL policy eliminates the cold-start period of SUPE while the advantage gate closes the exploitation gap above the pretrained baseline.

\subsection{AntMaze: Long-Horizon Sparse Navigation}
\label{subsec:antmaze}

We evaluate on antmaze-medium-play-v2 and antmaze-large-play-v2 from D4RL~\cite{fu2020d4rl}, which require long-horizon sparse-reward navigation through a maze to a fixed goal. These environments test whether the advantage gate concentrates raw-policy control near the goal region while the latent policy handles exploration through the maze body.

\spaars-SUPE achieves comparable performance to native SUPE on antmaze-medium-play-v2, reaching a normalized return of approximately 0.9. Gate activation heatmaps (Fig.~\ref{fig:antmaze_trajs}) confirm that $\pi_{\mathrm{raw}}$ fires predominantly in goal-proximal states, consistent with Proposition~\ref{prop:gate_convergence}.

\begin{figure}[h]
\centering
\includegraphics[width=0.48\textwidth]{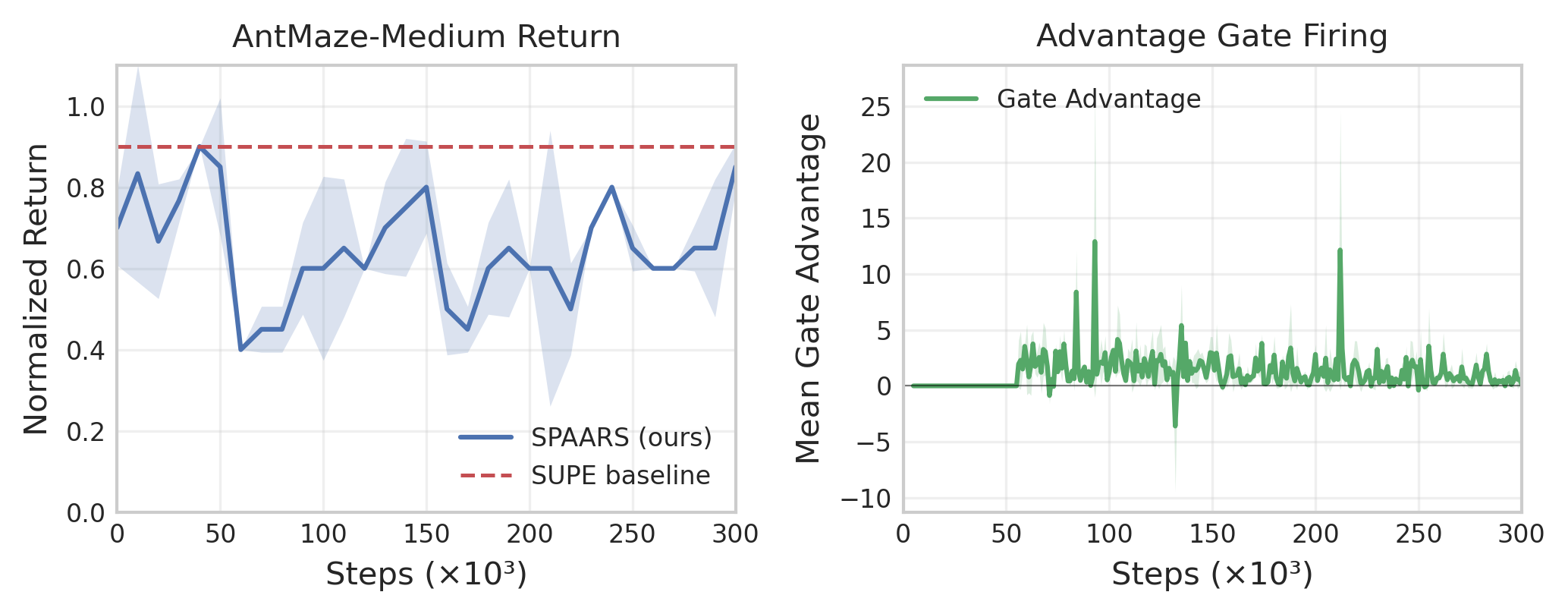}
\caption{AntMaze-Medium (left): \spaars-SUPE matches native SUPE performance. Advantage Gate Firing (right): Mean advantage of raw policy over latent policy increases as training progresses, enabling precise goal-reaching.}
\label{fig:antmaze}
\end{figure}

\begin{figure}[h]
\centering
\begin{subfigure}[b]{0.15\textwidth}
    \centering
    \includegraphics[width=\textwidth]{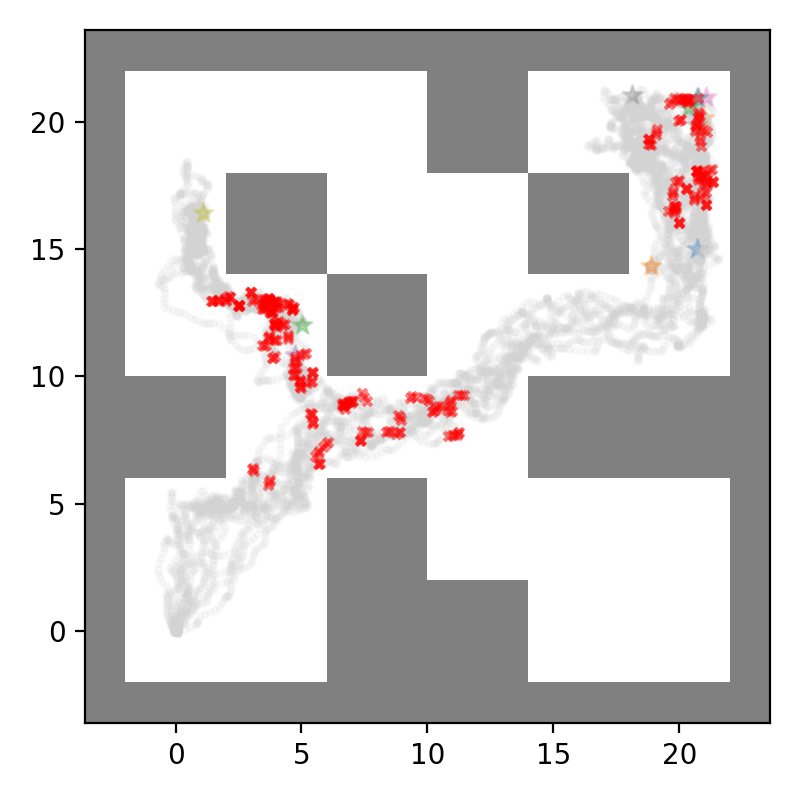}
    \caption{70k steps}
\end{subfigure}
\hfill
\begin{subfigure}[b]{0.15\textwidth}
    \centering
    \includegraphics[width=\textwidth]{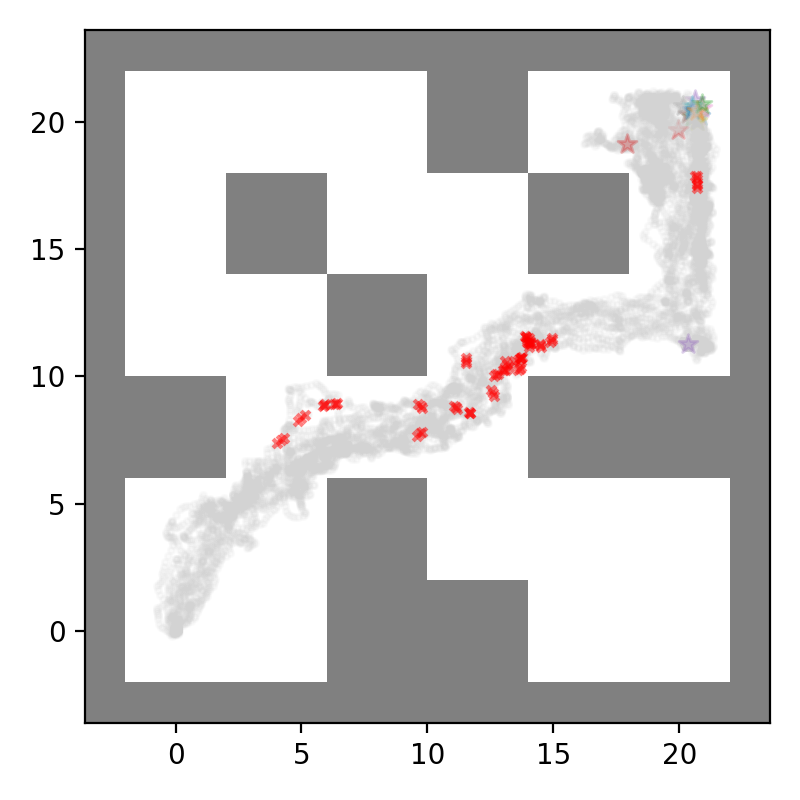}
    \caption{140k steps}
\end{subfigure}
\hfill
\begin{subfigure}[b]{0.15\textwidth}
    \centering
    \includegraphics[width=\textwidth]{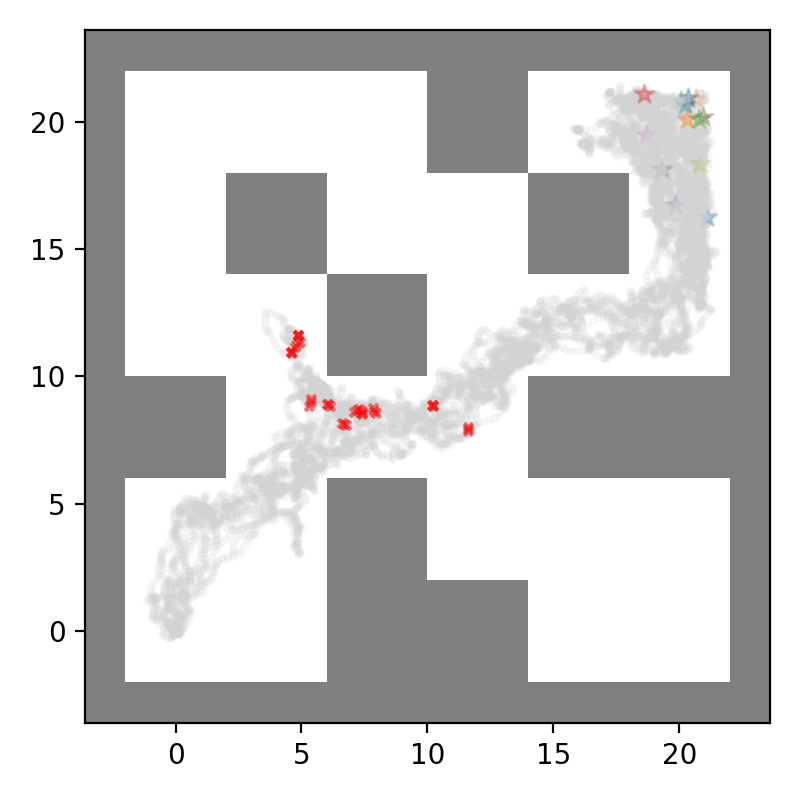}
    \caption{280k steps}
\end{subfigure}
\caption{Gate firing patterns over training. Red crosses indicate states where the Advantage Gate selects $\pi_{\mathrm{raw}}$. As training progresses, the gate successfully isolates raw-policy execution to the goal region while utilizing the latent policy for general maze exploration.}
\label{fig:antmaze_trajs}
\end{figure}


\subsection{Standalone \spaars: D4RL Locomotion}
\label{subsec:locomotion}

To validate the CVAE-based instantiation independently, we evaluate standalone \spaars\ on hopper-medium-v2 and walker2d-medium-v2, where the CVAE is trained exclusively on shuffled $(s,a)$ pairs with no trajectory structure. This directly tests the claim of Remark~\ref{rmk:bc_data}: that unordered pairs are sufficient for the online fine-tuning phase to improve upon the offline baseline. We compare against the offline IQL baseline reported by~\citet{kostrikov2021offline}.

\paragraph{Results.}
On hopper-medium-v2, standalone \spaars\ (schedule) warm-starts at $\approx$55 normalized return from the CVAE-initialized policy, dips during the $\alpha$-transition as control shifts from the latent to the raw policy, then recovers to \textbf{92.7$\pm$3.7 normalized return} by 300k steps (mean$\pm$95\% CI over 3 seeds), surpassing the IQL baseline of 66.3. On walker2d-medium-v2, \spaars\ warm-starts at $\approx$63 normalized return and reaches \textbf{102.9$\pm$9.4 normalized return}, surpassing both the IQL baseline of 78.3 and the expert demonstration performance. In both environments, the final performance exceeds the offline baseline significantly, confirming that the CVAE latent space---despite being learned without temporal ordering---captures sufficient action structure to bootstrap effective online exploration.

\begin{table}[h]
\centering
\caption{Standalone \spaars\ (schedule) vs.\ offline IQL on D4RL locomotion tasks. Normalized return at 300k steps, mean$\pm$95\% CI over 3 seeds.}
\label{tab:locomotion}
\begin{tabular}{lcc}
\toprule
Environment & IQL~\cite{kostrikov2021offline} & \spaars\ (ours) \\
\midrule
hopper-medium-v2   & 66.3 & \textbf{92.7$\pm$3.7} \\
walker2d-medium-v2 & 78.3 & \textbf{102.9$\pm$9.4} \\
\bottomrule
\end{tabular}
\end{table}

\begin{figure}[h]
\centering
\includegraphics[width=\linewidth]{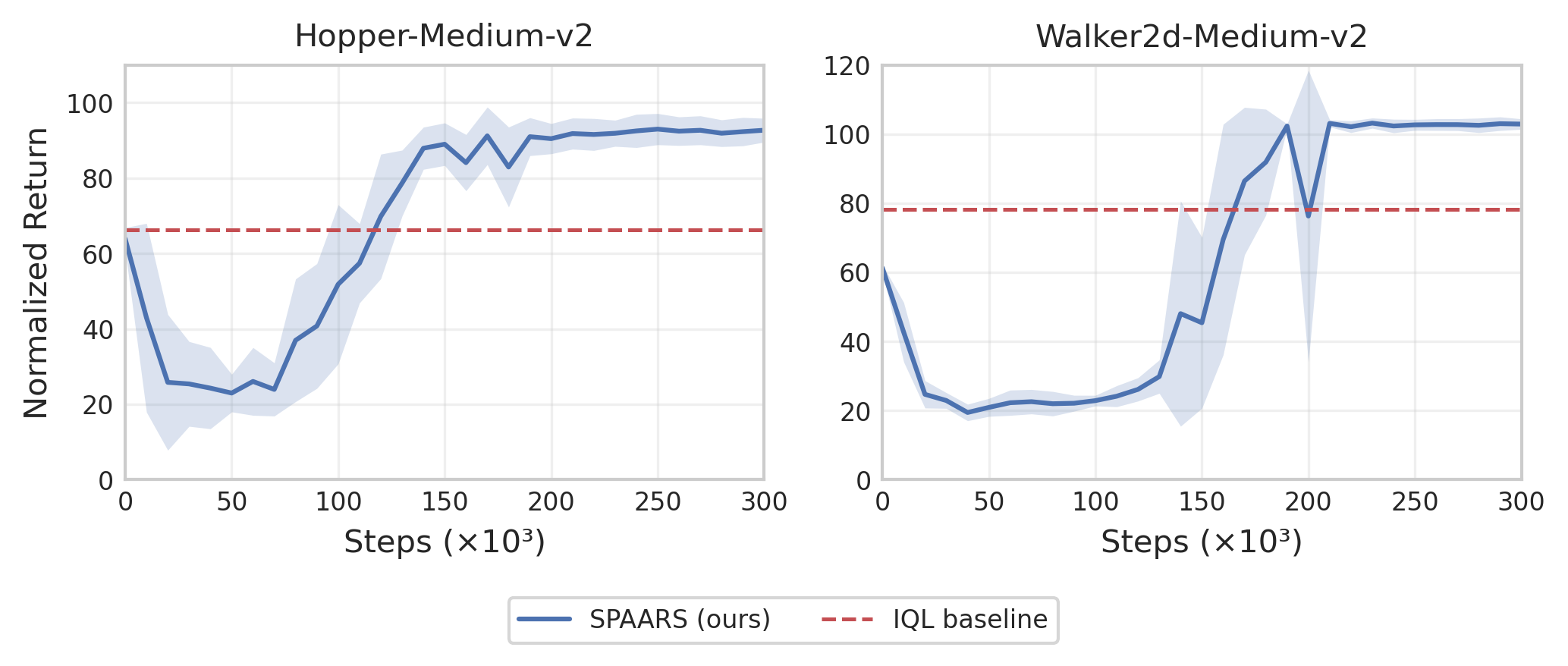}
\caption{Normalized return vs.\ environment steps on hopper-medium-v2 (left) and walker2d-medium-v2 (right). Standalone \spaars\ exceeds the offline IQL baseline on both tasks, validating that a CVAE trained on unordered $(s,a)$ pairs provides an effective latent manifold for online fine-tuning. Shaded regions indicate 95\% confidence intervals over 3 seeds.}
\label{fig:locomotion}
\end{figure}
\section{Conclusion}
\label{sec:conclusion}

In this work, we characterized the fundamental exploitation ceiling---bounded by decoder reconstruction error---that limits offline-to-online reinforcement learning agents operating in compressed latent skill spaces. We presented \spaars, a framework that bridges this gap by fusing the constraint-safe exploration of latent manifolds with raw-action exploitation. Crucially, we introduced an advantage-gated mode selection mechanism that replaces global $\alpha$ schedules with per-state decisions grounded in the Option-Critic termination gradient theorem. The gate uses the shared critic to activate the raw policy only where it demonstrably outperforms the decoder, preserving the latent policy's temporal abstraction for long-horizon navigation indefinitely. This eliminates the catastrophic forgetting of latent skills that plagues schedule-based approaches, yielding the best of both worlds: safe, structured exploration where it matters and unconstrained precision where it is needed. Through experiments spanning manipulation (kitchen-mixed-v0), long-horizon sparse navigation (AntMaze), and a standalone CVAE evaluation (D4RL locomotion), we demonstrated that \spaars-SUPE exceeds SUPE's asymptotic performance while achieving 5$\times$ better sample efficiency, and that the standalone CVAE instantiation surpasses the offline IQL baseline using only unordered state-action pairs.

\bibliographystyle{IEEEtran}
\bibliography{references}

\appendix

\section{Extended Theoretical Analysis}
\label{sec:appendix_theory}

This appendix provides extended discussion of theoretical claims, assumptions, and practical considerations.

\subsection{CVAE Convergence and Posterior Collapse}
\label{sec:cvae_convergence}

A practical prerequisite for SPAARS is that the CVAE training converges to a useful representation --- specifically, that the latent code $z$ carries meaningful information about the action $a$.

\begin{assumption}[CVAE Convergence]
\label{ass:cvae}
The CVAE avoids posterior collapse: $I(z; a \mid s) > \varepsilon_{\mathrm{info}}$ for some $\varepsilon_{\mathrm{info}} > 0$, where $I$ is mutual information under the learned encoder $q_\phi(z \mid s, a)$.
\end{assumption}

When this fails (posterior collapse), the decoder ignores $z$, rendering latent-space optimization ineffective.

\paragraph{Mitigation Strategies.}
\begin{itemize}
    \item \textbf{$\beta$-annealing:} Start with $\beta = 0$ in the KL term, gradually increase to $\beta \leq 1$.
    \item \textbf{Batch normalization of $\mu_z$:} Prevents mode collapse (as in ELAPSE~\cite{zhou2020plas}).
    \item \textbf{Free bits:} Allow a minimum number of nats per dimension before the KL penalty activates.
\end{itemize}

\subsection{Full Proof: Variance Reduction (Prop.~\ref{prop:variance})}
\label{sec:proof_variance}

\begin{proof}
The REINFORCE gradient is:
\begin{equation}
    \nabla_\theta J \!=\! \mathbb{E}_{s,a}\!\bigl[Q(s,a) \nabla_\theta \log \pi(a \mid s)\bigr].
\end{equation}

For isotropic Gaussian $\pi(a \mid s) = \mathcal{N}(\mu_\theta(s), \sigma^2 I_n)$, the score is $\nabla_\theta \log \pi = \sigma^{-2}(a - \mu_\theta)^\top \nabla_\theta \mu_\theta$. Its second moment satisfies:
\begin{equation}
    \mathbb{E}\!\bigl[\|\nabla_\theta \log \pi\|^2\bigr] = \tfrac{n}{\sigma^2} \|\nabla_\theta \mu_\theta\|^2.
\end{equation}

Using the variance decomposition:
\begin{align}
    \mathrm{Var}[\nabla_\theta J] &\approx \mathbb{E}[Q^2] \mathrm{Var}[\nabla_\theta \log \pi] \nonumber\\
    &\quad + \mathbb{E}[\|\nabla_\theta \log \pi\|^2] \mathrm{Var}[Q].
\end{align}

Comparing $n\!=\!k$ (latent) with $n\!=\!d$ (raw) gives the $k/d$ factor. The Q-variance ratio captures that restricting to $\mathcal{M}_a$ excludes high-variance OOD actions.
\end{proof}

\subsection{Full Proof: Exploitation Gap (Prop.~\ref{prop:expgap})}
\label{sec:proof_expgap}

\begin{proof}
\textbf{Step 1.} By the Performance Difference Lemma~\cite{kakade2002approximately}, for policies $\pi, \pi'$:
\begin{equation}
    J(\pi') - J(\pi) = \tfrac{1}{1-\gamma} \mathbb{E}_{s \sim d_{\pi'}}\!\bigl[\mathbb{E}_{a \sim \pi'}[A^\pi(s,a)]\bigr]
\end{equation}
where $A^\pi = Q^\pi - V^\pi$ and $d_{\pi'}$ is the discounted state visitation.

\textbf{Step 2.} For any state $s$, let $a^* = \pi_a^*(s)$, $\tilde{a} = \decode(\encode(a^*,s),s)$, and $z^* = \arg\max_z Q^*(s, \decode(z,s))$.

Since $z^*$ is optimal in $\mathcal{Z}$: $Q^*(s, \decode(z^*,s)) \geq Q^*(s, \tilde{a})$.
By the Lipschitz property:
\begin{align}
    &Q^*(s, a^*) - Q^*(s, \decode(z^*,s)) \nonumber\\
    &\quad \leq Q^*(s, a^*) - Q^*(s, \tilde{a}) \leq L_Q \|a^* \!-\! \tilde{a}\|.
\end{align}

\textbf{Step 3.} Aggregating over $d_{\pi_a^*}$:
\begin{align}
    \Delta_{\mathrm{exploit}} &= J(\pi_a^*) - J(\pi_z^*) \nonumber\\
    &\leq \tfrac{L_Q}{1-\gamma} \mathbb{E}_{s \sim d_{\pi_a^*}}\!\bigl[\|a^* \!-\! \tilde{a}\|\bigr] = \tfrac{L_Q \varepsilon_{\mathrm{rec}}^*}{1-\gamma}
\end{align}
where $\varepsilon_{\mathrm{rec}}^* = \mathbb{E}_{s \sim d_{\pi_a^*}}[\|a^*(s) - \decode(\encode(a^*,s),s)\|]$. Under Assumption~\ref{ass:coverage}, $\varepsilon_{\mathrm{rec}}^* \approx \varepsilon_{\mathrm{rec}}$.
\end{proof}

\subsection{Tightness of the $\frac{1}{1-\gamma}$ Factor}
\label{sec:gamma_tightness}

For long-horizon tasks ($\gamma = 0.99$, horizon $= 100$), the bound can be loose. In practice it is tighter because: (i)~the global $L_Q$ overestimates local sensitivity around optimal actions; (ii)~well-trained CVAEs achieve $\varepsilon_{\mathrm{rec}} < 0.1$; and (iii)~the per-state gap concentrates on states where the optimal action is far from $\mathcal{M}_a$.

\subsection{Full Proof: State-Dependent Regret Bound (Prop.~\ref{prop:gate_regret})}
\label{sec:proof_gate_regret}

\begin{proof}
Define the oracle policy $\pi^*_{\mathrm{oracle}}(s) = \arg\max\{Q^*(s, \pi_{\mathrm{raw}}(s)),\; Q^*(s, \decode(\pi_z(s), s))\}$, which always selects the truly better option. The gated policy $\pi_g$ uses the learned critic $Q$ instead of $Q^*$:
\begin{equation}
    \pi_g(s) = \arg\max\{Q(s, \pi_{\mathrm{raw}}(s)),\; Q(s, \decode(\pi_z(s), s))\}.
\end{equation}

At any state $s$, let $a_g = \pi_g(s)$ and $a^*_o = \pi^*_{\mathrm{oracle}}(s)$. The per-state suboptimality is:
\begin{align}
    Q^*(s, a^*_o) - Q^*(s, a_g) &\leq |Q^*(s, a^*_o) - Q(s, a^*_o)| \nonumber\\
    &\quad + |Q(s, a_g) - Q^*(s, a_g)|
\end{align}
since $Q(s, a_g) \geq Q(s, a^*_o)$ by definition of $\pi_g$ (it maximizes the learned $Q$). Each term is bounded by $\varepsilon_Q(s)$, giving a per-state gap of at most $2\varepsilon_Q(s)$. Applying the Performance Difference Lemma and absorbing constants:
\begin{equation}
    J(\pi^*_{\mathrm{oracle}}) - J(\pi_g) \leq \frac{1}{1-\gamma} \mathbb{E}_{s \sim d_{\pi_g}}[\varepsilon_Q(s)].
\end{equation}

The key insight is that this bound is \emph{independent of schedule design}. The $\alpha$ schedule introduces an additional regret term $\frac{1}{1-\gamma}\mathbb{E}_s[\mathbf{1}\{g^*(s)=z\} \cdot (Q^*(s, a_z) - Q^*(s, a_{\mathrm{raw}}))]$ from states where it forces the wrong option. The gate eliminates this term entirely: as $\varepsilon_Q \to 0$, $\pi_g \to \pi^*_{\mathrm{oracle}}$.
\end{proof}

\subsection{Gate Convergence Analysis}
\label{sec:gate_convergence}

The gate activation set $\mathcal{S}_{\mathrm{raw}} = \{s : \bar{Q}_{\mathrm{raw}}(s) - \bar{Q}_z(s) > m,\; \sigma_{\mathrm{raw}}(s) < \sigma_{\max}\}$ evolves during training. Early in training (large $\varepsilon_Q$, large $\sigma_{\mathrm{raw}}$ everywhere), the disagreement filter $\sigma_{\max}$ ensures $\mathcal{S}_{\mathrm{raw}} \approx \emptyset$---the gate defaults to $\pi_z$. As the critic converges:

\begin{enumerate}
    \item Ensemble disagreement $\sigma_{\mathrm{raw}}(s)$ decreases in well-visited states.
    \item The advantage estimate $\bar{Q}_{\mathrm{raw}} - \bar{Q}_z$ converges to the true exploitation gap $\Delta_{\mathrm{exploit}}(s)$.
    \item $\mathcal{S}_{\mathrm{raw}}$ converges to $\{s : \Delta_{\mathrm{exploit}}(s) > m\}$---precisely the states where the decoder ceiling is binding.
\end{enumerate}

In maze navigation tasks, this corresponds to states near the goal where precise actions (stopping, turning at the right moment) matter more than temporal coherence. Far from the goal, $\pi_z$'s $H$-step committed skills provide superior trajectory-level coherence, and the gate correctly preserves them.

\subsection{Comparison with PLAS}
\label{sec:plas_comparison}

PLAS~\cite{zhou2020plas} adds a bounded perturbation: $a = \decode(z,s) + \xi_\phi(s)$. This has two limitations vs.\ SPAARS: (i)~the effective action space is $\mathcal{M}_a \oplus B_\epsilon$, a small tube around the manifold, so actions far from $\mathcal{M}_a$ remain inaccessible; (ii)~no curriculum means early perturbations may push actions off-manifold before the critic is calibrated. SPAARS resolves both: $\pi_{\mathrm{raw}}$ can output \emph{any} $a \in \mathcal{A}$ when $\alpha=1$, and the curriculum ensures critic calibration (Prop.~\ref{prop:calibration}).

\end{document}